\documentclass[runningheads]{llncs}
\usepackage[T1]{fontenc}
\usepackage{graphicx}
\usepackage{amsmath,amssymb}
\usepackage{tikz}
\usepackage{hyperref}
\usepackage{cleveref}
\usepackage{algorithm}
\usepackage[noend]{algpseudocode}
\usepackage{algorithmicx}
\usepackage{subcaption}
\usepackage{booktabs}
\usetikzlibrary{arrows.meta,positioning,calc,fit,decorations.pathreplacing}

\begin{document}
\setlength{\textfloatsep}{8pt}
\setlength{\intextsep}{8pt}
\setlength{\floatsep}{6pt}
\setlength{\abovecaptionskip}{2pt}
\setlength{\belowcaptionskip}{0pt}

\title{Judgelight: Trajectory-Level Post-Optimization for Multi-Agent Path Finding\\ via Closed-Subwalk Collapsing}
\titlerunning{Judgelight: A Post-Optimization Layer for Learning-Based MAPF}
%
\author{Yimin Tang\inst{1} \and
Sven Koenig\inst{1,2} \and
Erdem B{\i}y{\i}k\inst{1}}
\authorrunning{Tang et al.}
%
\institute{Thomas Lord Department of Computer Science, University of Southern California\\
\email{\{yimintan,biyik\}@usc.edu}\\
\and
Department of Computer Science, University of California Irvine\\
\email{svenk@uci.edu}\\
}
\maketitle              
\begin{abstract}
Multi-Agent Path Finding (MAPF) is an NP-hard problem with applications in warehouse automation and multi-robot coordination. Learning-based MAPF solvers offer fast and scalable planning but often produce feasible trajectories that contain unnecessary or oscillatory movements. We propose \textbf{Judgelight}, a post-optimization layer that improves trajectory quality after a MAPF solver generates a feasible schedule. Judgelight collapses closed subwalks in agents' trajectories to remove redundant movements while preserving all feasibility constraints. We formalize this process as \textbf{MAPF-Collapse}, prove that it is NP-hard, and present an exact optimization approach by formulating it as integer linear programming (ILP) problem. Experimental results show Judgelight consistently reduces solution cost by around 20\%, particularly for learning-based solvers, producing trajectories that are better suited for real-world deployment.

\end{abstract}
\section{Introduction}

Multi-Agent Path Finding (MAPF) is an NP-hard problem~\cite{stern2019multi,yu2013structure} that plans collision-free trajectories for multiple agents moving from given start locations to distinct goal locations on a known graph while optimizing a cost criterion. MAPF is a fundamental abstraction for coordination in domains such as warehouse automation, transportation, and aerial swarms. Lots of search-based methods are designed to solve this problem, including Conflict-Based Search (CBS)~\cite{sharon2015conflict}, $M^*$~\cite{wagner2011m}, LaCAM~\cite{okumura2023lacam}, and MAPF-LNS2~\cite{li2022mapf}. These methods provide strong correctness guarantees and high-quality solutions, but their runtime increase substantially as the number of agents and the level of congestion grow, motivating complementary approaches that deliver fast decisions at scale.

Driven by the success of deep neural networks in perception, control, and decision making~\cite{krizhevsky2012imagenet,silver2016mastering,achiam2023gpt}, learning-based MAPF solvers have attracted increasing attention~\cite{alkazzi2024comprehensive}. Some learning-based methods, such as RAILGUN~\cite{tang2025railgununifiedconvolutionalpolicy}, adopt a \emph{centralized} formulation that uses global information as neural network input features and outputs the actions for all agents simultaneously. Others follow a \emph{decentralized} formulation: each agent maps a local observation (often a limited field-of-view, FOV) to a single-step action, and all agents act either synchronously or in a fixed order. Representative methods include PRIMAL~\cite{sartoretti2019primal}, MAPPER~\cite{liu2020mapper}, MAGAT~\cite{li2021message}, SCRIMP~\cite{wang2023scrimp}, MAPF-GPT~\cite{andreychuk2024mapf}, and SILLM~\cite{jiang2025deployingthousandrobotsscalable}. These approaches typically rely on imitation learning (IL) and/or reinforcement learning (RL).

Learned policies face two persistent challenges that limit their direct use as drop-in MAPF solvers. First, MAPF is safety-critical: even a single invalid action might cause a collision or a deadlock. In practice, learned policies are commonly wrapped by lightweight validity checks, and invalid moves are often patched by forcing one or more agents to wait. Recent work such as CS-PIBT~\cite{veerapaneni2024improvinglearntlocalmapf} helps learning-based methods avoid producing collision-inducing next actions by using neural network output probabilities to guide the search order in the search-based algorithm PIBT~\cite{okumura2022priority}, thereby effectively increasing success rates under congestion. While these corrections are useful as they may prevent collisions, they only focus on feasibility at the \emph{next-action} level. When the trajectories are already feasible, their post-processing do not improve the solution quality.

Second, even when a learned solver outputs feasible trajectories, the resulting motion may exhibit \emph{unnecessary movements} (e.g., back-and-forth movements)~\cite{veerapaneni2024improvinglearntlocalmapf} that are undesirable in real deployments. In warehouse settings, such redundant motion increases energy consumption and mechanical wear, and may elevate operational risk by introducing avoidable interactions with other robots and environmental uncertainties. Importantly, this inefficiency might persist even when no nearby agents constrain the motion~\cite{veerapaneni2024improvinglearntlocalmapf}, indicating that trajectory quality issues are not solely attributable to collision resolution.

In this work, we propose \textbf{Judgelight}, a post-optimization layer that improves trajectory quality \emph{after} a MAPF solver (learning-based or search-based) produces a feasible schedule. Judgelight targets redundant motion by allowing the solver to \emph{collapse closed subwalks} in individual agents' trajectories while maintaining all MAPF feasibility constraints. We formalize this post-processing task as a new optimization problem, \textbf{MAPF-Collapse}, whose objective is to minimize the total number of move actions (treating waits as low-cost) subject to feasibility. This formulation captures the deployment-motivated preference that an agent should stay in place unless movement is necessary. Our main contributions are:
\begin{enumerate}
    \item We introduce \textbf{MAPF-Collapse}, a trajectory-level post-optimization problem that takes a feasible MAPF schedule as input and reduces redundant movements through collapse operations.
    \item We analyze the computational complexity of MAPF-Collapse and prove that it is \textbf{NP-hard}.
    \item We develop an exact optimization approach, \textbf{Judgelight}, for MAPF-Collapse and demonstrate that it consistently improves solution quality, particularly for learning-based MAPF solvers, producing trajectories that are more suitable for real-world warehouse-style operations.
\end{enumerate}

\section{Background: MAPF Problem Definition}
The MAPF problem is defined as follows: Let $I=\{1,2,\cdots,N\}$ denote a set of $N$ agents. $G = (V,E)$ represents an undirected graph, where each vertex $v \in V$ represents a possible location of an agent in the workspace, and each edge $e \in E$ is a unit-cost edge between two vertices that moves an agent from one vertex to the other. Self-loop edges are also allowed, which represent ``wait-in-place'' actions.
Each agent $i\in I$ has a start location $s_i \in V$ and a goal location $g_i \in V$. It also holds that $s_i \neq s_j$ and $g_i \neq g_j$ when $i \neq j$ $\forall i, j \in I$.
The task is to plan a collision-free path for each agent $i$ from $s_i$ to $g_i$.

Each action of agents, either waiting in place or moving to an adjacent vertex, takes one time unit.
Let $v^i_t \in V$ be the location of agent $i$ at timestep $t$.
Let $\pi_i=[v_0^{i}, v_1^{i}, ..., v_{T^{i}}^{i}]$ denote a path of agent $i$ from its start location $v_0^{i}$ to its target $v_{T^{i}}^{i}$. 
We assume that agents rest at their targets after completing their paths, i.e., $v_t^i = v_{T^i}^i, \forall t > T^i$.
We refer to the path with the minimum cost as the shortest path.

We consider two types of agent-agent collisions.
The first type is \emph{vertex collision}, where two agents $i$ and $j$ occupy the same vertex at the same timestep. 
The second type is \emph{edge collision}, where two agents move in opposite directions along the same edge simultaneously.
We use $(i, j, t)$ to denote a vertex collision between agents $i$ and $j$ at timestep $t$ or an edge collision between agents $i$ and $j$ at timestep $t$ to $t+1$.
MAPF algorithms usually use \emph{SoC (flowtime)} $\sum_{i=1}^{N}T^{i}$ as the cost function.

Overall, we need to find a set of paths $\{\pi_i \mid i\in I\}$ for all agents such that, for each agent $i$:
\begin{enumerate}
\item Agent $i$ starts from its start location (i.e., $v_0^{i} = s_i$) and stops at its target location $g_j$ (i.e., $v_{t}^{i} = g_j, \forall t \ge T^{i}$).
\item Every pair of adjacent vertices on path $\pi_i$ is connected by an edge, i.e., $(v_{t}^{i}, v_{t+1}^{i}) \in E, \forall t \in \{0,1,\dots,T^i\}$.
\item $\{\pi_i \mid i\in I\}$ is collision-free.
\end{enumerate}

\section{Related Work}

\subsection{Multi-Agent Path Finding (MAPF)}

MAPF has been proved an NP-hard problem when optimality is required~\cite{yu2013structure}. It has inspired a wide range of solutions for its related challenges. 
Decoupled strategies, as outlined in \cite{silver2005cooperative,wang2008fast,luna2011push}, approach the problem by independently planning paths for each agent before integrating these paths. 
In contrast, coupled approaches \cite{standley2010finding,standley2011complete} devise a unified plan for all agents simultaneously. There also exist dynamically coupled methods~\cite{sharon2015conflict,wagner2015subdimensional} that consider agents planning independently at first and then together only when needed for resolving agent-agent collisions. 
Among these, Conflict-Based Search (CBS) algorithm \cite{sharon2015conflict} stands out as a centralized and optimal method for MAPF, with several bounded-suboptimal variants such as ECBS~\cite{barer2014suboptimal} and EECBS~\cite{li2021eecbs}. Some suboptimal MAPF algorithms, such as Prioritized Planning (PP)~\cite{erdmann1987multiple,silver2005cooperative}, PBS~\cite{ma2019searching}, LaCAM~\cite{okumura2023lacam} and their variant methods~\cite{chan2023greedy,li2022mapf,okumura2023lacam3} exhibit better scalability and efficiency. However, these search-based algorithms always face the problem of search space dimensionality explosion as the problem size increases, making it difficult to produce a feasible solution within a limited time. Learning-based methods, which we discuss in the subsequent section, attempt to overcome the dimensionality issue by learning from large amounts of data or experience.

\subsection{Learning-based MAPF}

Given the success of deep learning, many learning-based MAPF methods have been proposed. Compared to search-based algorithms, these methods can typically produce plans in a short time and automatically learn heuristic functions. Most approaches use IL, RL, or a combination of both. Learning-based solvers can make end-to-end decisions using all available information and can be trained on various data types (e.g., MAPF and LMAPF). In contrast, search-based methods often require multi-stage decompositions with hand-crafted heuristics.

One early learning-based method for MAPF is PRIMAL~\cite{sartoretti2019primal}, which is trained using both RL and IL. It is a decentralized algorithm that relies on a local field of view (FOV) around each agent to generate actions. MAPF-GPT~\cite{andreychuk2024mapf} is a GPT-based model for MAPF, trained via IL on a large dataset. Other approaches incorporate communication mechanisms in a decentralized manner, such as GNN~\cite{li2020graph} and MAGAT~\cite{li2021message}, which employ graph neural networks for inter-agent communication, and SCRIMP~\cite{wang2023scrimp}, which uses a global communication mechanism based on transformers. SILLM~\cite{jiang2025deployingthousandrobotsscalable} applies IL and achieves higher planning speed and throughput than search-based methods in certain LMAPF scenarios. RAILGUN~\cite{tang2025railgununifiedconvolutionalpolicy} adopts a centralized IL formulation and outputs actions for all agents simultaneously. DeepFleet~\cite{agaskar2025deepfleetmultiagentfoundationmodels} is trained on real warehouse data and improves the efficiency of Amazon warehouse operations by 10\%. Recent work such as CS-PIBT~\cite{veerapaneni2024improvinglearntlocalmapf} improves learned local policies by combining them with PIBT~\cite{okumura2022priority} at the action-distribution level, effectively increasing success rates under congestion. However, these corrections primarily operate at the \emph{next-action} level and focus on collision avoidance; they cannot detect whether an agent is actually making progress toward its goal or merely exhibiting oscillatory movement.

\subsection{Lifelong MAPF}

A generalization of the MAPF problem, Lifelong MAPF (LMAPF) continuously assigns new target locations to agents once they have reached their current targets. In LMAPF, agents do not need to arrive at their targets simultaneously. There are three main approaches to solving LMAPF: solving the problem as a whole~\cite{nguyen2019generalized}, using MAPF methods but replanning all paths at each specified timestep~\cite{li2021lifelong,okumura2022priority}, and replanning only when agents reach their current targets and are assigned new ones~\cite{vcap2015complete,grenouilleau2019multi}. Recently, PIBT~\cite{okumura2022priority} has played a key role in LMAPF. Its core idea is that higher-priority agents can push lower-priority agents to alternative locations, enabling progress toward their next targets and ensuring completeness. Similar to standard MAPF, LMAPF solutions may also exhibit unnecessary or oscillatory movements, and our method \textbf{Judgelight} can be naturally extended to this setting.


\section{MAPF-Collapse}

In this section, we first formulate the MAPF post-optimization problem, which we call \textbf{MAPF-Collapse}. We then prove that MAPF-Collapse is NP-hard.

We define the Multi-Agent Path Finding collapse (MAPF-Collapse) problem as follows.
Let $I=\{1,2,\cdots,N\}$ denote a set of $N$ agents.
Let $G=(V,E)$ be an undirected graph, where each vertex $v\in V$ represents a possible location of an agent, and each edge $e\in E$ is a unit-cost edge between two vertices.
As in standard MAPF, we allow self-loop edges, which represent \emph{wait-in-place} actions.

Each agent $i\in I$ is associated with a discrete-time trajectory over a fixed time horizon $T\in\mathbb{N}$.
The input is given as a matrix \(M \in V^{N\times (T+1)}\), 
where $M[i,t]\in V$ denotes the location of agent $i$ at timestep $t$, for
$t=0,1,\ldots,T$.
We write $v_t^i = M[i,t]$ for notational consistency with MAPF. We assume the given matrix $M$ satisfies all constraints in the MAPF problem definition.

\paragraph{Allowed Modification (Collapse): }
For any agent $i\in I$ and any indices $0 \le a < b \le T$ such that \( v_a^i = v_b^i \), 
we may replace the subsequence: 
\[
[v_a^i, v_{a+1}^i, \ldots, v_b^i]
\]
by the constant sequence \([x, x, \ldots, x], \text{where } x = v_a^i = v_b^i\). This operation is called a \emph{collapse of a closed subwalk}.
Multiple collapses may be applied to different agents independently.

Let $M'$ denote the modified schedule after applying some collapses, and
let $v_t^{\prime i} := M'[i,t]$. The modified schedule $M'$ must satisfy the same constraints that mentioned in MAPF problem definition: Each agent finally stops at target locations, each pair of adjacent vertices must be connected and the path solution is collision-free.

\paragraph{Cost Function: }
Each action (waiting or moving) takes one time unit.
Traversing an edge incurs unit cost, while waiting has lower cost. In this paper, we assume it to be zero.
The total cost of a modified schedule $M'$ is defined as
\[
\mathrm{cost}(M') =
\sum_{i=1}^{N}
\left|
\left\{
t \in \{0,\ldots,T-1\}
\;\middle|\;
v_t^{\prime i} \neq v_{t+1}^{\prime i}
\right\}
\right|.
\]

The objective of the MAPF-Collapse problem is to find a set of collapses applied to the input schedule $M$ such that the resulting schedule $M'$ is feasible and minimizes $\mathrm{cost}(M')$. Although MAPF-Collapse is defined based on the standard MAPF formulation, it can be readily extended to other MAPF variants, such as Task and Path Finding (TAPF)~\cite{ma2016optimal,tang2023solving,tang2024ita} and LMAPF~\cite{li2021lifelong,okumura2022priority}. For example, given an LMAPF solution, one can remove all collapse actions associated with completed goals at their corresponding timesteps and then apply the same MAPF-Collapse formulation to optimize the remaining trajectories.

\subsection{MAPF-Collapse NP-hardness}
\label{subsec:collapsed_mapf_nphard}

We prove that \textsc{MAPF-Collapse} is NP-hard via its decision version. Decision MAPF-Collapse $(G, M, \beta)$: Given a graph $G=(V,E)$, a horizon $T\in\mathbb{N}$, a feasible input schedule
$M\in V^{N\times (T+1)}$, and a bound $\beta\in\mathbb{N}$, decide whether there exists
a sequence of collapse operations applied to $M$ that produces a feasible schedule $M'$
such that $\mathrm{cost}(M')\le \beta$.

\begin{lemma}\label{lem:collapsed_in_np}
\textsc{Decision MAPF-Collapse} is in NP.
\end{lemma}
\begin{proof}
A certificate can be the resulting schedule $M'$ (of size $O(NT)$).
We verify feasibility by checking (i) connectivity for every agent and timestep,
(ii) absence of vertex collisions, and (iii) absence of edge swaps, all in $O(NT)$ time.
We compute $\mathrm{cost}(M')$
in $O(NT)$ time and compare it with $\beta$.
Thus verification runs in polynomial time.
\end{proof}

\begin{theorem}\label{thm:collapsed_npcomplete}
\textsc{Decision MAPF-Collapse} is NP-complete. \textsc{MAPF-Collapse}
(the optimization version) is NP-hard.
\end{theorem}
\begin{proof}
We reduce from \textsc{Independent Set}: given an undirected graph $H=(U,F)$ and
an integer $K$, decide whether there exists an independent set $S\subseteq U$ with
$|S|\ge K$.

\medskip
\noindent\textbf{Construction}

Let $m:=|F|$ and fix an indexing $F=\{e_r \mid r=1,\dots,m\}$.
We construct an instance $(G,M,\beta)$ of \textsc{Decision MAPF-Collapse} with
$N:=|U|+|F|$ agents.

\smallskip
\noindent\emph{Vertices.}
For each $u\in U$ in \textsc{Independent Set}, create two vertices $x_u$ and $p_u$ in $G$.
For each edge $e\in F$, create three private vertices $a_e,b_e,c_e$ in $G$.
All created vertices are distinct.

\smallskip
\noindent\emph{Edges.}
Add $\{x_u,p_u\}$ to $G$ for all $u\in U$.
For each edge $e=(u,v)\in F$, add the following undirected edges to $G$
(waiting is allowed by staying at the same vertex between consecutive timesteps, as in the problem definition):
\[
\{a_e,x_u\},\ \{x_u,b_e\},\ \{b_e,c_e\},\ \{c_e,a_e\},\ \{a_e,x_v\},\ \{x_v,b_e\}.
\]

\smallskip
\noindent\emph{Agents and schedule $M$.}
We create one \emph{vertex-agent} $A_u$ for each $u\in U$ and one \emph{edge-agent}
$B_{e_r}$ for each $e_r\in F$, so the total number of agents is $N:=|U|+|F|$. Set the horizon (we subtract one because timesteps start from zero): $T = 7m-1$.

\begin{itemize}
\item For each vertex-agent $A_u$:
\[
M[A_u,0]=x_u,\quad M[A_u,t]=p_u\ (1\le t\le T-1),\quad M[A_u,T]=x_u.
\]
\item For each edge-agent $B_{e_r}$, define a block start time $\Theta_r := 7r-7$.
Outside its block, $B_{e_r}$ waits at a fixed private vertex: it stays at $a_{e_r}$
for all $t\leq\Theta_r$, and stays at $b_{e_r}$ for all $t \geq \Theta_r+6$.
Within the block, set
\[
\bigl(M[B_{e_r},\Theta_r],\dots,M[B_{e_r},\Theta_r+6]\bigr)
=
\bigl(a_{e_r},\ x_{u_r},\ b_{e_r},\ c_{e_r},\ a_{e_r},\ x_{v_r},\ b_{e_r}\bigr).
\]
\end{itemize}

\begin{figure*}[t!]
\centering
\includegraphics[width=0.7\textwidth]{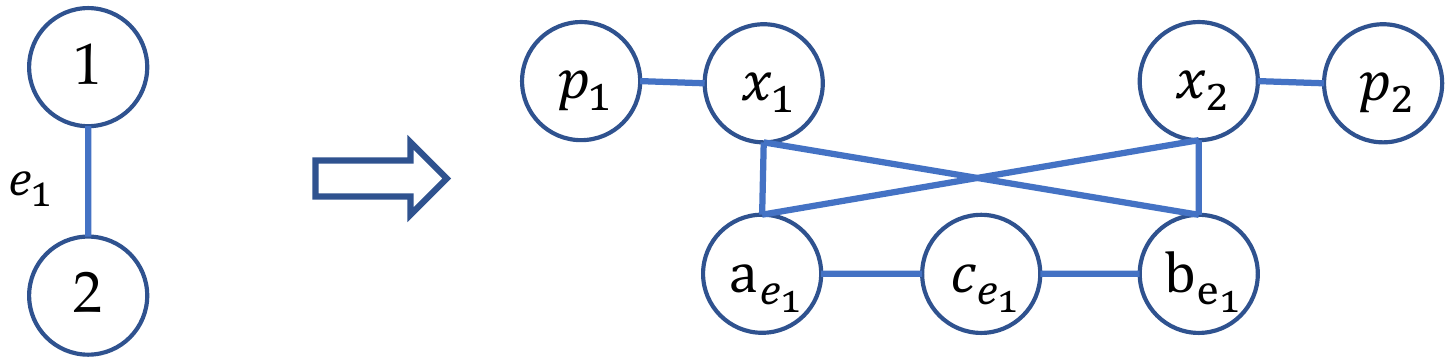}
\caption{
We create three agents $A_1$, $A_2$, and $B_{e_1}$.
The given paths are
$A_1$: \{$x_1$, $p_1$, $p_1$, $p_1$, $p_1$, $p_1$, $x_1$\},
$A_2$: \{$x_2$, $p_2$, $p_2$, $p_2$, $p_2$, $p_2$, $x_2$\}, and
$B_{e_1}$: \{$a_{e_1}$, $x_1$, $b_{e_1}$, $c_{e_1}$, $a_{e_1}$, $x_2$, $b_{e_1}$\}.
In this example, the optimal collapsed schedule collapses either $\{A_1, B_{e_1}\}$ or
$\{A_2, B_{e_1}\}$.
}
\label{fig:2points}
\end{figure*}

In Figures~\ref{fig:2points} and \ref{fig:3points}, we present two small examples to provide intuition. These examples illustrate how we transform a graph with two or three vertices from the \textsc{Independent Set} problem into an instance of \textsc{MAPF-Collapse}.

\smallskip
\noindent\emph{Feasibility of $M$.}
Because (i) the blocks $\{\Theta_r,\dots,\Theta_r+6\}$ are disjoint across $r$,
(ii) all vertices $a_e,b_e,c_e$ are private to their corresponding edge-agent, and
(iii) each vertex-agent $A_u$ stays at the private vertex $p_u$ for times $1,\dots,T-1$,
the input schedule $M$ is feasible.

\begin{figure*}[t!]
\centering
\includegraphics[width=0.8\textwidth]{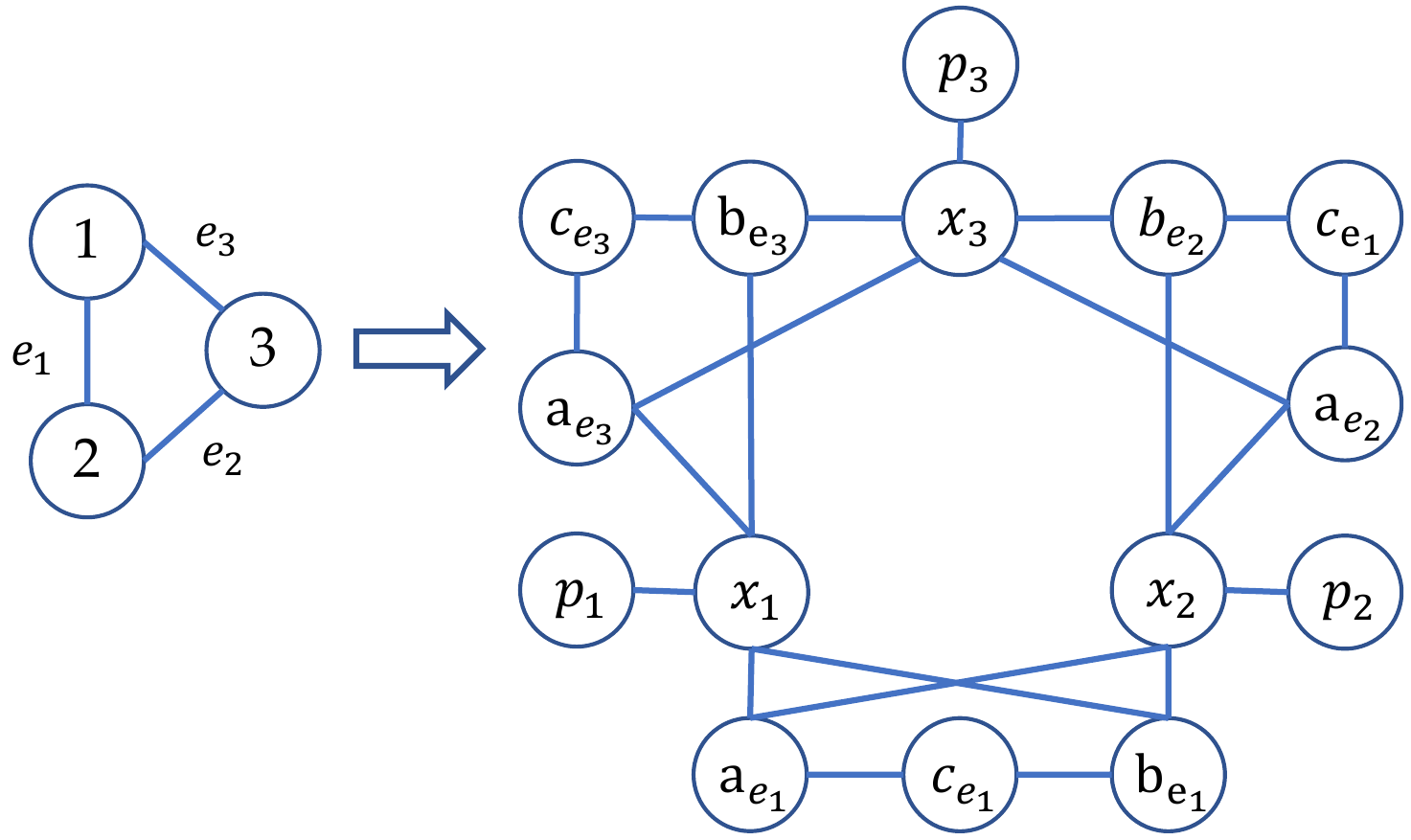}
\caption{
We create six agents $A_1$, $A_2$, $A_3$, $B_{e_1}$, $B_{e_2}$, and $B_{e_3}$.
The given paths are
$A_1$: \{$x_1$, $p_1$, \ldots, $p_1$, $x_1$\},
$A_2$: \{$x_2$, $p_2$, \ldots, $p_2$, $x_2$\},
$A_3$: \{$x_3$, $p_3$, \ldots, $p_3$, $x_3$\},
$B_{e_1}$: \{$a_{e_1}$, $x_1$, $b_{e_1}$, $c_{e_1}$, $a_{e_1}$, $x_2$, $b_{e_1}$, \ldots, $b_{e_1}$\},
$B_{e_2}$: \{$a_{e_2}$, \ldots, $a_{e_2}$, $x_2$, $b_{e_2}$, $c_{e_2}$, $a_{e_2}$, $x_3$, $b_{e_2}$, \ldots, $b_{e_2}$\}, and
$B_{e_3}$: \{$a_{e_3}$, \ldots, $a_{e_3}$, $x_1$, $b_{e_3}$, $c_{e_3}$, $a_{e_3}$, $x_3$, $b_{e_3}$\}.
All agents' paths have length 21. In this example, we can collapse at most one vertex-agent $A_i$.
However, regardless of which $A_i$ is collapsed, each edge-agent $B_{e_r}$ can always be collapsed once.
The only difference is that the choice of $A_i$ determines which collapse option is feasible for the incident edge-agents.
}
\label{fig:3points}
\end{figure*}

\smallskip
\noindent\emph{Collapses and cost savings.}
For each $u\in U$, the subwalk of $A_u$ on $[0,T]$ is closed (from $x_u$ back to $x_u$).
Collapsing it makes $A_u$ stay at $x_u$ for all times and saves exactly $2$ moves
(the moves $x_u\to p_u$ and $p_u\to x_u$).
For each edge-agent $B_{e_r}$ (block start $\Theta=\Theta_r$), there are two possible collapses:
\[
[\Theta,\Theta+4]\ \text{(endpoints $a_{e_r}$)} \qquad [\Theta+2,\Theta+6]\ \text{(endpoints $b_{e_r}$)}.
\]
Each saves exactly $4$ moves (the four edge-traversals inside the collapsed interval),
and they cannot both be applied since they overlap on $[\Theta+2,\Theta+4]$.
Moreover, any other collapse of $B_{e_r}$ yields zero saving: outside the block $B_{e_r}$
only waits at a fixed vertex, and inside the block the only repeated vertices are
$a_{e_r}$ (at times $\Theta,\Theta+4$) and $b_{e_r}$ (at times $\Theta+2,\Theta+6$).
We also note that for $e_r=(u,v)$, even if we do not collapse $A_u$ or $A_v$, we can still
apply one of the above collapses to $B_{e_r}$.

\smallskip
\noindent\emph{Bound.}
Let $C_0:=\mathrm{cost}(M)$ and set
\[
\beta := C_0 - 4m - 2K.
\]
Since each edge-agent $B_{e_r}$ can be collapsed once regardless of which vertex-agents are collapsed,
we can always obtain a total saving of $4m$ from the edge-agents.

\medskip
\noindent\textbf{Correctness}

\smallskip
\noindent(\,$\Rightarrow$\,)
Let $S\subseteq U$ be an independent set with $|S|\ge K$.
For each $u\in S$, collapse $A_u$ on $[0,T]$ so that $A_u$ stays at $x_u$ for all times.
For each edge $e_r=(u_r,v_r)$, modify $B_{e_r}$ as follows:
collapse $[\Theta_{r},\Theta_{r}+4]$ if $u_r\in S$; collapse $[\Theta_{r}+2,\Theta_{r}+6]$ if $v_r\in S$;
otherwise apply either one (arbitrarily).
Since $S$ is independent, at most one endpoint of $e_r$ lies in $S$, hence $B_{e_r}$ is
never required to apply both $[\Theta_{r},\Theta_{r}+4]$ and $[\Theta_{r}+2,\Theta_{r}+6]$.
After these collapses, $B_{e_r}$ avoids the uniquely occupied endpoint (if any) in its block,
so no collisions arise. Therefore the resulting schedule $M'$ is feasible.
The total saving is at least $4m+2|S|\ge 4m+2K$, so
$\mathrm{cost}(M')\le C_0-(4m+2K)=\beta$.

\smallskip
\noindent(\,$\Leftarrow$\,)
Suppose there exists a feasible $M'$ with $\mathrm{cost}(M')\le \beta$.
Let
\[
S:=\{u\in U \mid A_u \text{ is collapsed on }[0,T]\}.
\]
Each edge-agent $B_{e_r}$ can save at most $4$ since $[\Theta_{r},\Theta_{r}+4]$ and
$[\Theta_{r}+2,\Theta_{r}+6]$ are mutually exclusive and no other collapse yields positive saving.
Thus, the total saving from all edge-agents is at most $4m$.
Because $\mathrm{cost}(M')\le \beta = C_0-4m-2K$, the vertex-agents together must contribute
saving at least $2K$, implying $|S|\ge K$.

It remains to show that $S$ is independent. Suppose for contradiction that some edge
$e_r=(u_r,v_r)\in F$ has $u_r,v_r\in S$. Then $A_{u_r}$ occupies $x_{u_r}$ at all times,
so $B_{e_r}$ must eliminate its visit to $x_{u_r}$ at time $\Theta_r+1$; the only
positive-saving collapse that changes time $\Theta_r+1$ is $[\Theta_{r},\Theta_{r}+4]$,
hence $[\Theta_{r},\Theta_{r}+4]$ must be applied. Similarly, since $A_{v_r}$ occupies
$x_{v_r}$ at all times, $B_{e_r}$ must eliminate its visit to $x_{v_r}$ at time $\Theta_r+5$,
which forces applying $[\Theta_{r}+2,\Theta_{r}+6]$.
This contradicts the fact that $[\Theta_{r},\Theta_{r}+4]$ and $[\Theta_{r}+2,\Theta_{r}+6]$
cannot both be applied. Hence $S$ is an independent set of size at least $K$.

\medskip
\noindent\textbf{NP-hardness Conclusion.}
Therefore, the graph $H$ has an independent set of size at least $K$ if and only if the constructed
instance admits a feasible solution of cost at most $\beta$.
The construction has size polynomial in $|U|+|F|$ and can be computed in polynomial time, so this is a
polynomial-time reduction. Together with \Cref{lem:collapsed_in_np}, it follows that
\textsc{Decision MAPF-Collapse} is NP-complete. NP-completeness of the decision version implies
NP-hardness of the optimization version.\qed
\end{proof}

\subsection{Judgelight}
\label{subsec:collapsed_mapf_ilp}
We refer to our whole proposed post-optimization process as \textbf{Judgelight}. Here we show how to formulate MAPF-Collapse as an \textbf{Integer Linear Programming} (ILP) problem and demonstrate how Judgelight can help MAPF solvers improve the quality of their final solutions.

Throughout, we assume the input schedule $M\in V^{N\times (T+1)}$ is collision-free, which is already given by the problem definition.
A \emph{collapse action} is a tuple $c=(i,a,b,x)$ where $i\in [N]$, $0\le a<b\le T$, $x=M[i,a]=M[i,b]$, and applying $c$
rewrites all locations of agent $i$ on timesteps $k\in[a,b]$ to $x$.

\paragraph{Candidate actions and weights.}
Let the set of all candidate actions be
\[
\mathcal{C}\;:=\;\bigl\{\,c=(i,a,b,x)\ \big|\ 0\le a<b\le T,\ x=M[i,a]=M[i,b]\,\bigr\}.
\]
For each $c=(i,a,b,x)\in\mathcal{C}$, define its weight
\[
w_c\;=\;\sum_{k=a}^{b-1}\mathbb{I}\bigl[M[i,k]\neq M[i,k+1]\bigr],
\]
i.e., the number of edge traversals on $[a,b]$ in the original schedule. Collapsing $[a,b]$ turns all those traversals into waits,
hence saves exactly $w_c$ cost.

\paragraph{Decision variables and objective function.}
For each candidate $c\in\mathcal{C}$, introduce a binary variable $y_c\in\{0,1\}$ indicating whether $c$ is applied. We maximize the total saving:
\[
\textbf{maximize}\quad \sum_{c\in\mathcal{C}} w_c\,y_c.
\]
Under the non-overlap constraints below, the savings contributed by selected actions on the same agent do not double-count, and we have
$\mathrm{cost}(M')=\mathrm{cost}(M)-\sum_{c\in\mathcal{C}} w_c\,y_c$. Therefore this objective is equivalent to minimizing
$\mathrm{cost}(M')$.

\paragraph{Constraints from relations between actions.}
Selecting actions is constrained by relations induced from $M$.

\smallskip
\noindent\emph{(1) Mutual exclusion within an agent.}
Within one agent, there may exist multiple candidate collapse actions. If two actions
$c=(i,a,b,x)$ and $c'=(i,a',b',x')$ have intersecting time ranges, i.e.,
\[
[a,b]\cap[a',b']\neq \emptyset,
\]
then we select at most one of them, since otherwise the endpoints of one action would be rewritten by the other action, making the action invalid.
This yields the constraint
\[
y_{c} + y_{c'} \leq 1.
\]

\smallskip
\noindent\emph{(2) Mutual exclusion across agents (action--action collisions).}
Two actions on different agents may force them to occupy the same vertex at overlapping timesteps.
For $c=(i,a,b,x)$ and $c'=(j,a',b',x')$ with $i\neq j$, if $x=x'$ and $[a,b]\cap[a',b']\neq\emptyset$,
then selecting both actions would create an unavoidable vertex collision. Hence we add the constraint
\[
y_c + y_{c'} \le 1
\]

\begin{algorithm}[t]
\caption{Build ILP relations from $M$}
\label{alg:find_relation_short}
\small
\begin{algorithmic}[1]
\Require Feasible schedule $M\in V^{N\times(T+1)}$
\Ensure Candidate actions $\mathcal{C}$ with weights $\{w_c\}$; within-agent exclusions $\mathcal{E}_{\mathrm{in}}$;
cross-agent exclusions $\mathcal{E}_{\mathrm{cross}}$; dependencies $\mathcal{D}$; invalid actions $\mathcal{I}$

\State Build $\mathrm{Occ}[k][x]:=\{j\in[N]\mid M[j,k]=x\}$ for all $k\in[0..T], x\in V$
\State For each agent $i$, precompute prefix move counts $\mathrm{Pref}_i$ where $\mathrm{Pref}_i[0]=0$, 
$\mathrm{Pref}_i[t+1]=\mathrm{Pref}_i[t]+\mathbb{I}[M[i,t]\neq M[i,t+1]]$

\State $\mathcal{C}\gets \emptyset$
\For{$i=1$ to $N$}
  \State Build timestep lists $\mathrm{Times}_i[x]:=\{k\mid M[i,k]=x\}$ for visited vertices $x$
  \ForAll{$x$ with $\mathrm{Times}_i[x]\neq\emptyset$}
    \ForAll{$(a,b)$ such that $a,b\in\mathrm{Times}_i[x]$ and $a<b$}
      \State Add $c=(i,a,b,x)$ to $\mathcal{C}$ with $w_c=\mathrm{Pref}_i[b]-\mathrm{Pref}_i[a]$
    \EndFor
  \EndFor
\EndFor

\State $\mathcal{E}_{\mathrm{in}}\gets$ all pairs of $c,c'\in\mathcal{C}$ on the same agent with $[a,b]\cap[a',b']\neq\emptyset$
\State $\mathcal{E}_{\mathrm{cross}}\gets$ all pairs of $c,c'\in\mathcal{C}$ on different agents with $x=x'$ and $[a,b]\cap[a',b']\neq\emptyset$

\State Build an Segment Tree $\mathsf{Idx}_j$ for each agent $j$ over its actions in $\mathcal{C}$, which allows us to efficiently retrieve all actions related to a given position.
\State $\mathcal{D}\gets \emptyset,\ \mathcal{I}\gets \emptyset$
\ForAll{$c=(i,a,b,x)\in\mathcal{C}$}
  \For{$k=a$ to $b$}
    \ForAll{$j\in \mathrm{Occ}[k][x]$ with $j\neq i$}
      \State $\mathcal{S}\gets \{c'=(j,a',b',x')\in \mathsf{Idx}_j.\mathrm{Query}(k)\mid x'\neq x\}$
      \If{$\mathcal{S}=\emptyset$} \State $\mathcal{I}\gets \mathcal{I}\cup\{c\}$ \State \textbf{break $k$ and $j$ loops} \EndIf
      \State $\mathcal{D}\gets \mathcal{D}\cup\{(c,\mathcal{S})\}$ \Comment{$y_c\le \sum_{c'\in\mathcal{S}} y_{c'}$}
    \EndFor
  \EndFor
\EndFor

\State \Return $(\mathcal{C},\{w_c\},\mathcal{E}_{\mathrm{in}},\mathcal{E}_{\mathrm{cross}},\mathcal{D},\mathcal{I})$
\end{algorithmic}
\end{algorithm}

\smallskip
\noindent\emph{(3) Dependency across agents (collision avoidance).}
A collapse may force an agent to occupy a vertex $x$ over a time interval; to keep the schedule feasible, other agents that would occupy the
same vertex at those timesteps must be ``moved away'' by selecting some collapse action on those agents.
For example, if we have two agents $\{1, 2\}$ and the paths are $(A, B, A)$ and $(C, A, C)$, then there are two collapse actions
$c=(1, 0, 2, A)$ and $c'=(2, 0, 2, C)$, and selecting $c$ requires that agent $2$ also selects a suitable action so that collisions are avoided.

Formally, for each selected action $c=(i,a,b,x)$, for every other agent $j\neq i$, for all $k\in[a,b]$ such that $M[j,k]=x$,
define the set of \emph{suitable} actions on agent $j$:
\[
\mathcal{C}(j,k,x)
\;:=\;
\bigl\{
c'=(j,a',b',x')\in\mathcal{C}
\ \big|\
k\in[a',b']\ \text{and}\ x'\neq x
\bigr\}.
\]
Then collision avoidance is enforced by the implication
\[
y_c \le \sum_{c'\in \mathcal{C}(j,k,x)} y_{c'} 
\]
\[
\forall\, c=(i,a,b,x)\in\mathcal{C},\ \forall\, j\neq i,\ \forall\, k\in[a,b]\ \text{with}\ M[j,k]=x
\]


Therefore, the remaining task is to systematically extract (i) mutual-exclusion relations,
(ii) dependency relations, and (iii) invalid actions from the input schedule $M$. We summarize the constraint construction procedure in \Cref{alg:find_relation_short}. We also note that Judgelight can be applied to any MAPF algorithm, not only learning-based methods, although its benefits are most prominent over learning-based solvers due to the inherent uncertainty of neural network policies. Moreover, in theory, Judgelight can be extended to other MAPF variants, such as those with rotation actions as well as TAPF and LMAPF, which are more representative of real-world warehouse scenarios.

In practice, when forming the ILP, we apply two essential optimizations:
\begin{enumerate}
\item For any path segment of the form ``AA...AA'', in principle every pair of identical
locations could be used to form a collapse action. However, all such actions are
equivalent in terms of both cost savings and interactions with other agents. Therefore,
we keep only the collapse defined by the first and last occurrences of ``A'' when
constructing ILP variables. For example, for a path ``AAABAAA'', there are originally
$\binom{6}{2}$ possible collapse actions, but after this optimization, only
$\binom{4}{2}$ actions remain. As a result, the number of ILP variables and constraints
is significantly reduced without affecting the optimality of \textsc{MAPF-Collapse}.

  \item For oscillatory segments such as ``ABABABA'', we first apply a length-3 filter
  that rewrites every subsequence of the form ``ABA'' into ``AAA'' whenever this
transformation can be applied without introducing collisions. This optimization is
  necessary because, in practice, frequent back-and-forth motion can cause the number of
  ILP variables to become extremely large. For example, an alternating sequence
  ``ABABAB\ldots'' of length $n$ can generate $O(n^2)$ variables and $O(n^4)$ constraints,
  and $n$ can be on the order of 100 depending on the map size. This preprocessing step
  can affect the optimality of the original problem, but it is necessary in practice.
\end{enumerate}

\begin{figure}[tb!]
    \centering
    \begin{subfigure}[b]{0.25\textwidth}
        \centering
        \includegraphics[width=\textwidth]{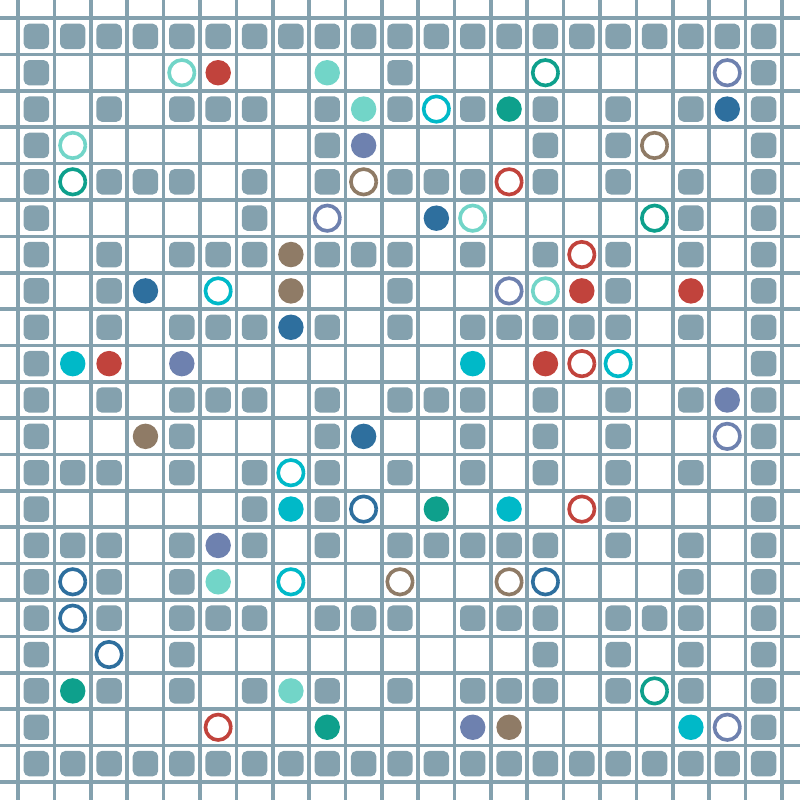}
        \caption{Maze}
        \label{fig:maze}
    \end{subfigure}
    \begin{subfigure}[b]{0.25\textwidth}
        \centering
        \includegraphics[width=\textwidth]{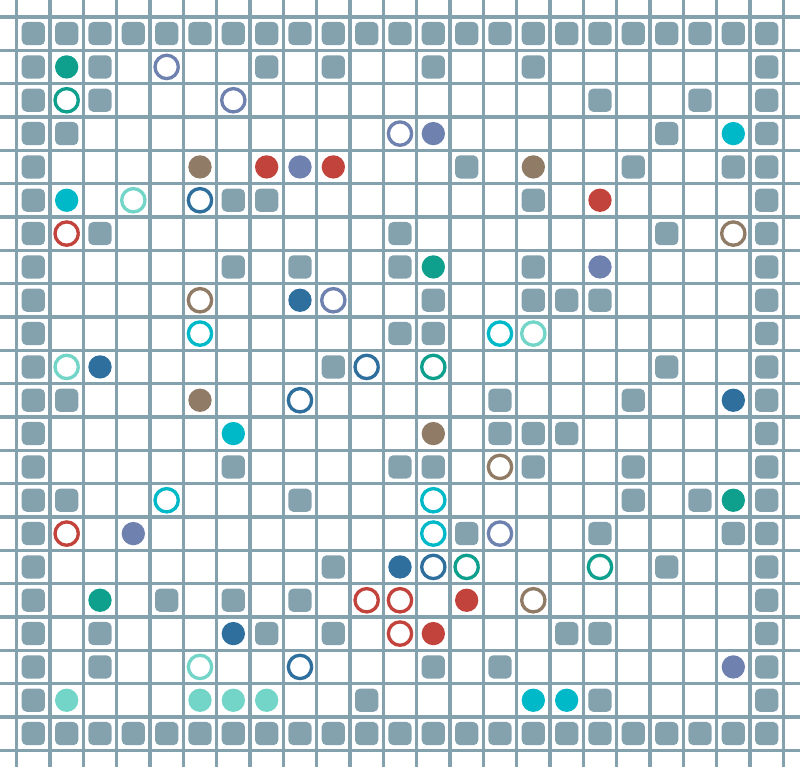}
        \caption{Random}
        \label{fig:random}
    \end{subfigure}
    \begin{subfigure}[b]{0.25\textwidth}
        \centering
        \includegraphics[width=\textwidth]{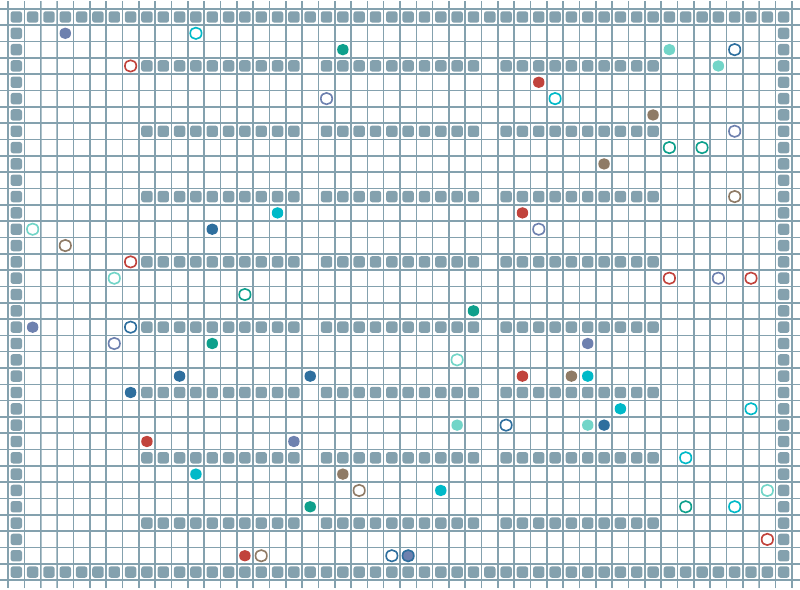}
        \caption{Warehouse}
        \label{fig:warehouse}
    \end{subfigure}
    
    \begin{subfigure}[b]{0.25\textwidth}
        \centering
        \includegraphics[width=\textwidth]{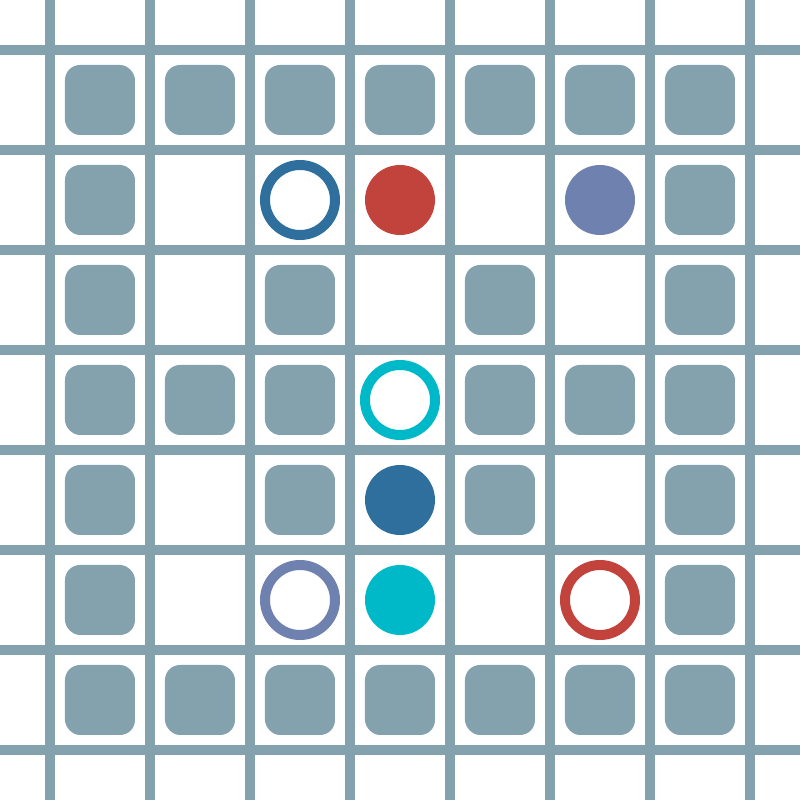}
        \caption{Puzzle}
        \label{fig:puzzle}
    \end{subfigure}
    \begin{subfigure}[b]{0.25\textwidth}
        \centering
        \includegraphics[width=\textwidth]{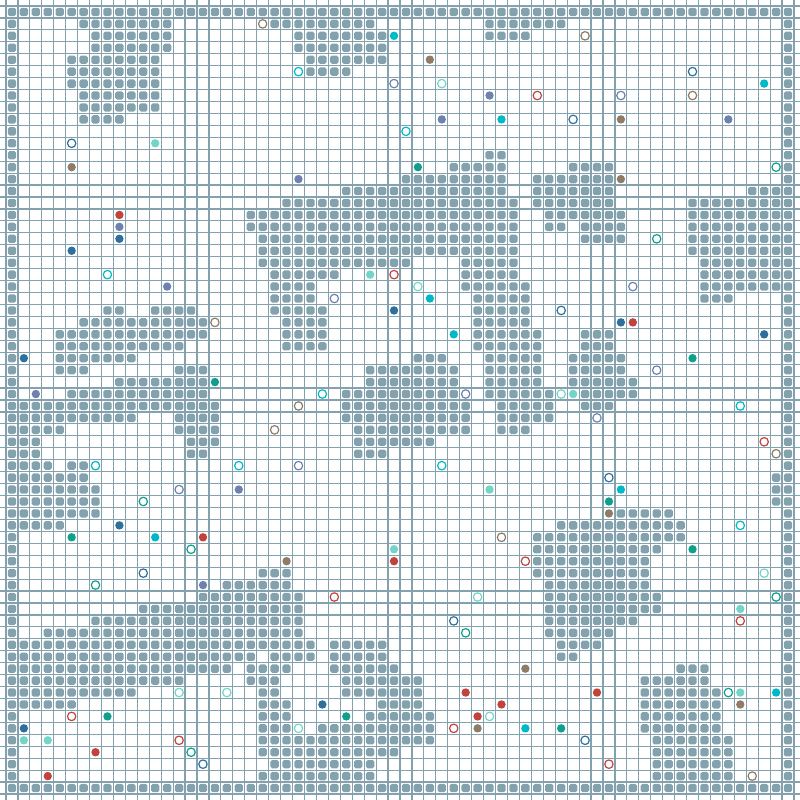}
        \caption{Cities-tiles}
        \label{fig:cities-tiles}
    \end{subfigure}
    \caption{Examples of POGEMA-tested maps. Maze and Random map sizes are 32x32. Warehouse map is 33x46. Puzzle map sizes are below 10x10. Cities-tiles are 64×64 areas selected from larger Cities maps.}
    \label{fig:maps_examples}
\end{figure}

\section{Experiments}

For experiments, we primarily evaluate our Judgelight post-optimization on learning-based methods. We compare solution quality before and after applying Judgelight, using solution cost as the main metric. We adopt the POGEMA benchmark~\cite{skrynnik2024pogemabenchmarkplatformcooperative} to evaluate learning-based methods with Judgelight.

All experiments\footnote{Our code will be released as open source after the camera-ready version.} were conducted on a system running Ubuntu~22.04.1~LTS, equipped with an Intel Core i9-12900K CPU, 128\,GB RAM, and an NVIDIA RTX~3080 GPU. We use the POGEMA benchmark with five different map types, as shown in \Cref{fig:maps_examples}, resulting in a total of 3{,}296 test cases. For the ILP solver, we use Gurobi~13.0. For the ILP solver, we impose a time limit of 5 seconds on solving, but do not limit the time spent on ILP construction.

\begin{figure*}[t!]
\centering
\includegraphics[width=0.9\textwidth]{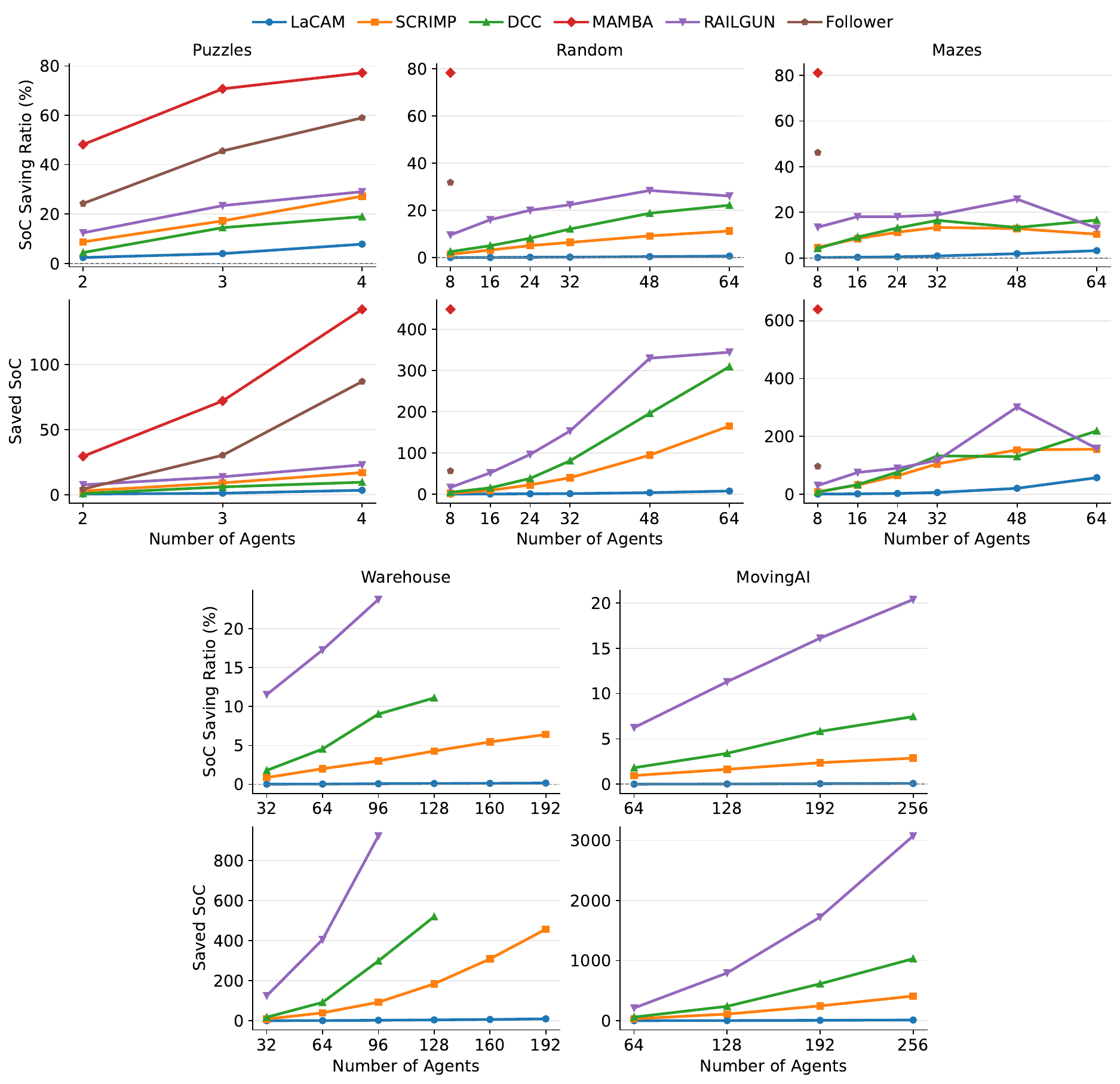}
\caption{This figure illustrates how Judgelight performance is affected by testcase difficulty and algorithm performance. We include only test cases for which the algorithm achieves $\mathrm{ISR}=1$, as Judgelight's cost savings are not meaningful when agents exhibit largely random motion (even though the apparent savings could be high in such cases). Missing data indicate that the algorithm has no suitable test cases under the corresponding experimental settings.}
\label{fig:1}
\end{figure*}

\begin{figure*}[t!]
\centering
\includegraphics[width=0.9\textwidth]{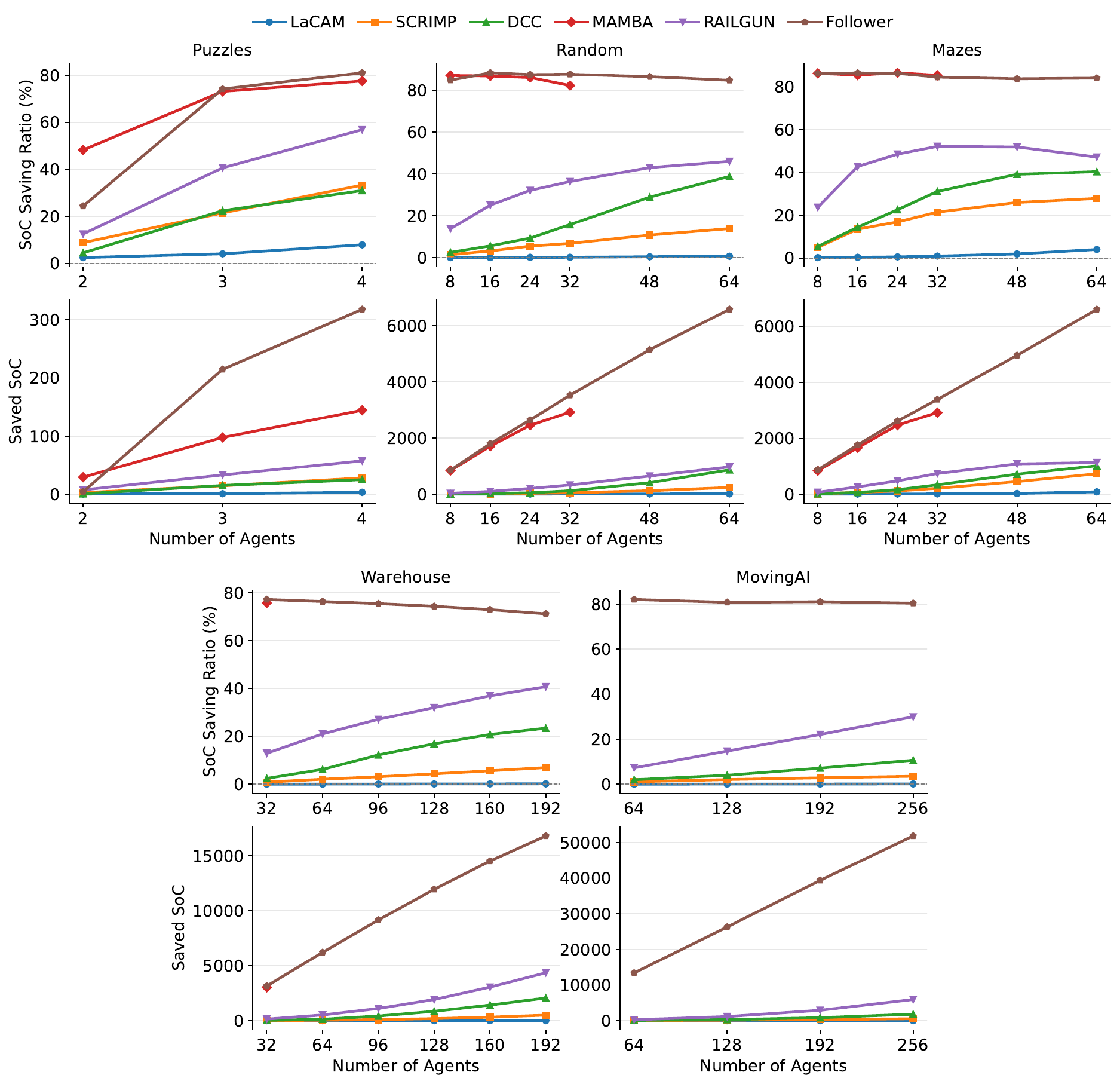}
\caption{In this figure, we include only test cases for which the algorithm achieves $\mathrm{ISR}>0.5$. This choice was made to provide more informative test cases. An $\mathrm{ISR}>0.5$ indicates that the learning-based policy is still partially effective on the test case. While such performance may not be sufficient for the MAPF task, it may still be useful for LMAPF.}
\label{fig:2}
\end{figure*}


In POGEMA, the learning-based methods include VDN~\cite{sunehag2018value}, QPLEX~\cite{wang2020qplex}, SCRIMP~\cite{wang2023scrimp}, IQL~\cite{tan1993multi}, QMIX~\cite{rashid2020monotonic}, DCC~\cite{ma2021learning}, MAMBA~\cite{egorov2022scalable}, Switcher~\cite{skrynnik2023switch}, Follower~\cite{skrynnik2024learn}, and MATS-LP~\cite{skrynnik2024decentralized}. However, based on the results reported in the RAILGUN paper~\cite{tang2025railgununifiedconvolutionalpolicy}, we evaluate SCRIMP, DCC, RAILGUN, MAMBA and Follower in this work, as they achieve good success rates on POGEMA. We also evaluate LaCAM as a representative suboptimal search-based method. 

Here we introduce several evaluation metrics. \textbf{ISR} (Individual Success Rate) is defined as the percentage of agents that eventually stop at their goal locations. \textbf{Saved SoC} denotes the reduction in SoC between the original solution and the solution obtained after applying Judgelight. \textbf{SoC Saving Ratio} is defined as Saved SoC divided by original solution SoC. \textbf{Agent density} is computed as the average, over all timesteps, of the number of agents within an agent's field of view (FOV) divided by the number of non-obstacle grid cells within that FOV. All FOV sizes use the same setting in POGEMA. The default FOV size is $11\times11$; for DCC, it is $9\times9$. For the centralized methods RAILGUN and LaCAM, we use the default FOV to compute agent density.

It is important to note that all tested methods are designed for the standard MAPF problem formulation. Since standard MAPF does not assign a lower cost to wait actions compared to move actions and learning-based policies are trained with standard MAPF data, the cost metric used in MAPF-Collapse is not perfectly aligned with the original objectives of some methods. As a result, the reported costs should be interpreted as illustrating the potential benefits that MAPF algorithms can gain from the Judgelight post-optimization under different scenarios, rather than as a direct comparison of their native optimization objectives.

\begin{figure*}[t!]
\centering
\includegraphics[width=0.9\textwidth]{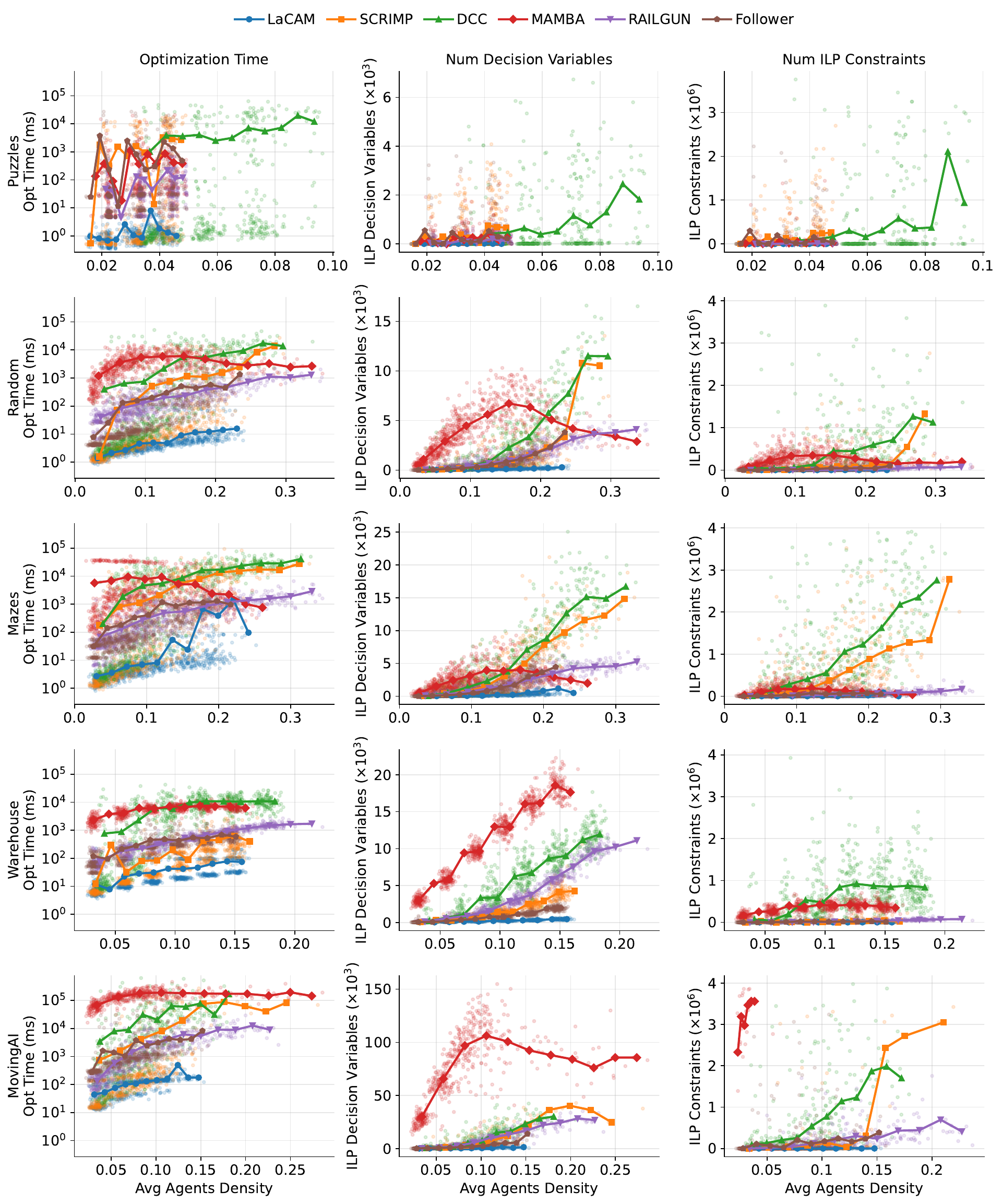} 
\caption{This figure shows the relations between agent density and ILP generated actions/constraints, and also agent density and Judgelight running time.This figure include all testcases even the method's ISR$=0$. In ILP constraints, we didn't show data points which larger than 4e6 in the figure for making all method can be distinguish.}
\label{fig:4}
\end{figure*}

In Figures~\ref{fig:1} and \ref{fig:2}, we report the total SoC savings achieved by Judgelight post-optimization for the MAPF problem. We observe a clear trend: as the number of agents increases, both the Saved SoC and the SoC saving ratio increase. For a fixed map type and map size, increasing the number of agents leads to higher agent density and a higher probability of path conflicts, making the test cases more difficult. As a result, suboptimal algorithms tend to explore multiple collision-avoidance behaviors. Since most learning-based methods predict only the next action, such exploratory behavior often manifests as oscillatory motion. After a collision-free solution is found, Judgelight can substantially reduce these redundant movements, achieving an average SoC reduction of approximately 20\%--40\%. As shown in \Cref{fig:2}, we also observe that MAMBA and Follower exhibit consistently high SoC saving ratios across different map types. This is because these methods often fail to fully solve the MAPF problem; for agents that do not reach their goal locations, the policies tend to produce oscillatory behavior, causing unnecessary movements to constitute a large fraction of the overall solution cost. This observation indicates that Judgelight can also be beneficial in partially solved scenarios by encouraging agents that cannot reach their goals under the current policy to remain stationary, rather than executing unsafe and unnecessary actions.

\Cref{fig:4} focuses on the runtime performance of Judgelight. In real-world warehouse operations or warehouse competition settings~\cite{chan2024league_robot_runners}, real-time planning is typically required, making Judgelight's runtime performance critical for practical deployment. As shown in \Cref{fig:4}, Judgelight's runtime is mainly influenced by two factors: agent density and algorithm performance. Higher agent density corresponds to more difficult test cases, and in such scenarios learning-based methods are more likely to generate unnecessary movements. This leads to a larger number of ILP variables and constraints, increasing both ILP construction time and solver time. We also observe that the number of ILP constraints has a greater impact on runtime than the number of variables. For example, in the warehouse scenario, although RAILGUN and DCC generate a similar number of ILP variables, DCC consistently exhibits longer runtimes because it produces more ILP constraints. Algorithm performance also strongly affects the size of the ILP, as shown in \Cref{tab:runtime_1s}. LaCAM, which has strong MAPF-solving capability, generates relatively few ILP actions and constraints, resulting in runtimes that are consistently below 100ms. For learning-based methods such as SCRIMP, DCC, and RAILGUN, which achieve relatively good performance, Judgelight's runtime is also below one second in most test cases.

\begin{table}[t!]
\centering
\begin{tabular}{l@{\hspace{10mm}}c@{\hspace{10mm}}c}
\toprule
 & $\mathrm{ISR}\geq 0$ & $\mathrm{ISR}=1$ \\
\midrule
LaCAM     & 3286 / 3296 (99.70\%) & 3266 / 3270 (99.88\%) \\
SCRIMP   & 2783 / 3296 (84.44\%) & 2657 / 2699 (98.44\%) \\
DCC      & 1724 / 3296 (52.31\%) & 1600 / 1799 (88.94\%) \\
MAMBA    & 880 / 3296  (26.70\%) & 226 / 230  (98.26\%) \\
RAILGUN  & 2699 / 3296 (81.89\%) & 1036 / 1041 (99.52\%) \\
Follower & 2910 / 3296 (88.29\%) & 110 / 110  (100.00\%) \\
\bottomrule
\end{tabular}
\caption{Judgelight runtime statistics measured by the proportion of test cases whose
optimization time is within 1 second.}
\label{tab:runtime_1s}
\end{table}
\vspace{-6mm}

\section{Conclusion}

Learning-based MAPF solvers often produce feasible solutions that exhibit unnecessary and oscillatory movements, especially under congestion or partial failures, due to their next-action–level decision making. We formalize the removal of such redundant motion as a trajectory-level post-processing problem, which we call \textbf{MAPF-Collapse}, prove that it is NP-hard, and develop an exact solution via an integer linear programming formulation. We refer to the resulting post-optimization framework as \textbf{Judgelight}. Experiments on the POGEMA benchmark show that Judgelight consistently reduces solution cost by around 20\%--40\% while remaining practical in runtime, with over 90\% of solved MAPF test cases finishing within one second. These results demonstrate that Judgelight is an effective, solver-agnostic post-processing step for improving MAPF trajectory quality in real-world settings. As future work, Judgelight could be extended to other multi-agent planning settings, and enhanced to eliminate not only redundant movements but also unnecessary waiting actions.

%
%
%

\newpage

\bibliographystyle{splncs04}
\bibliography{strings,myref}

\end{document}